\documentclass[10pt]{article}

\usepackage[margin=1in]{geometry}
\usepackage{times}
\usepackage{microtype}
\usepackage{graphicx}
\usepackage{xcolor}
\usepackage{booktabs}
\usepackage{algorithm}
\usepackage{algorithmic}
\usepackage{amsmath,amssymb}
\usepackage{amsthm}
\usepackage{aliascnt}
\usepackage{enumitem}
\usepackage[numbers,sort&compress]{natbib}
\usepackage[colorlinks=true,linkcolor=blue,citecolor=blue,urlcolor=blue]{hyperref}
\usepackage[capitalize,noabbrev]{cleveref}

\title{AC-JEPA: Anchored Cross-Modal Joint-Embedding Predictive Architecture for Visual-to-Symbolic Translation}

\author{
  {\normalsize Chaofei HONG\textsuperscript{*}, Ziyang Chen, Baochen Wang, Christina Yi Jin, Qing Zhang}\\[0.4em]
  {\normalsize Zhejiang Lab}\\[0.2em]
  {\small\texttt{hongchf@zhejianglab.com}}
}
\hypersetup{pdfauthor={Chaofei HONG, Ziyang Chen, Baochen Wang, Christina Yi Jin, Qing Zhang}}

\newcommand{\methodfull}{Anchored Cross-Modal JEPA}
\newcommand{\method}{AC-JEPA}
\newcommand{\ocsr}{Optical Chemical Structure Recognition}

\newaliascnt{theorem}{proposition}
\newtheorem{theorem}[theorem]{Theorem}
\aliascntresetthe{theorem}
\newaliascnt{corollary}{proposition}
\newtheorem{corollary}[corollary]{Corollary}
\aliascntresetthe{corollary}

\begin{document}

\maketitle

\begin{abstract}
Cross-modal translation requires models to capture shared semantics across modalities while preserving the fine-grained structure required by the target modality.
Yet in real-world scientific domains, paired data are often imperfect, while obtaining clean pairs at scale is costly.
Supervised fine-tuning (SFT) directly optimizes cross-modal translation through categorical token-level supervision, but this objective is sensitive to errors in the target sequences.
JEPA offers a complementary representation-space objective, but most existing formulations focus on prediction within a modality rather than translation across modalities.
We introduce \methodfull{} (\method{}), an SFT-compatible plug-in that predicts masked intermediate states in both cross-modal directions using paired data.
The joint-embedding objective targets cross-modal structure predictable from context, while the causal SFT path retains direct token-level supervision for the target sequence.
During self-training, \method{} learns from continuous latent targets derived from imperfect pseudo-labeled pairs, providing a route for using these pairs without promoting pseudo-tokens to hard categorical supervision.
Using optical chemical structure recognition as a testbed, gold-pair \method{} achieves 76.07\% accuracy on WildMol Simple, compared with 67.40\% for Base SFT.
Both models share complete MolParser-7M synthetic pretraining and Real-SFT fine-tuning; \method{} reuses the already-seen synthetic pairs through latent prediction without adding unique gold-pair data.
Using the same pseudo-labeled data for both approaches, \method{}-based self-training further raises accuracy to 80.23\%, whereas accuracy under direct SFT drops from 67.40\% to 63.47\%.
These results support SFT-anchored joint-embedding prediction as an effective strategy for cross-modal translation under both curated and imperfect supervision.

\end{abstract}

\section{Introduction}
\label{sec:introduction}
Scientific images encode knowledge through the identities, positions, and relations of graphical elements rather than through appearance alone.
Translating such images into symbolic or programming languages therefore requires a model to recognize local primitives, recover spatial or topological structure, resolve long-range dependencies, and serialize the result under a strict target grammar.
This differs from coarse semantic recognition or open-ended captioning: a single misplaced symbol can alter the encoded object or invalidate the output.
The problem is further complicated by real scientific data, where visual observations and paired sequences can be imperfect and expensive to curate at scale.
\ocsr{} (OCSR) exemplifies these demands.
Molecular diagrams encode atoms, bonds, stereochemistry, and topology through a spatial notation, whereas SMILES represents the same molecular structure as a compact, grammar-sensitive sequence~\citep{weininger1988smiles,fang2024molparser}.

Autoregressive multimodal generators provide a direct solution to this conversion problem.
A visual encoder extracts image evidence, and a causal language decoder combines it with the previously generated symbols to predict the target sequence.
Supervised fine-tuning (SFT) supplies exact token-level supervision and is therefore effective for learning the target grammar and decoding procedure~\citep{fang2024molparser}.
Its categorical objective, however, specifies the desired output tokens without directly requiring the fused intermediate states to recover predictable structure removed from either modality.
This distinction becomes important when training uses model-generated pairs.
Direct pseudo-label SFT promotes every generated token to a hard target, so each local pseudo-label error directly biases the decoder's learned conditional distribution~\citep{he2020revisitingselftraining,song2020noisylabels}.

The usefulness of paired data also depends on the stage and objective through which they are used.
Large synthetic corpora provide broad molecular coverage for pretraining, whereas real-image fine-tuning addresses the visual distribution of papers and patents~\citep{fang2024molparser}.
After this adaptation, simply resuming token supervision on synthetic pairs need not improve real-scene recognition: in our small-scale implementation, mixing MolParser-7M synthetic pairs with Real-SFT did not improve over Real-SFT-only adaptation and was more prone to overfitting.
This observation motivates extracting further structural constraints from already-seen synthetic pairs while preserving the real-data token anchor.

Representation-learning objectives such as contrastive alignment, masked reconstruction, and latent prediction provide complementary ways to learn correspondences and contextual structure across paired modalities~\citep{radford2021clip,he2022mae,assran2023ijepa}.
JEPA is particularly relevant because it predicts representations, focusing the objective on structure predictable from context rather than every input detail.
Most JEPA formulations, however, optimize encoders or embedding predictors; exact autoregressive translation between modalities is absent or remains separate from latent pretraining.
Applying latent prediction to a symbolic generator consequently requires more than adding a standalone pretraining loss: the predicted states must remain useful to the causal decoder, use evidence from both modalities, and avoid representational collapse.

We address these requirements with \methodfull{} (\method{}), an SFT-compatible latent predictive objective inside the same generator used for inference.
\Cref{fig:macro-flow} summarizes how the SFT anchor and latent objective interact.

\begin{figure}[t]
\centering
\includegraphics[width=\linewidth]{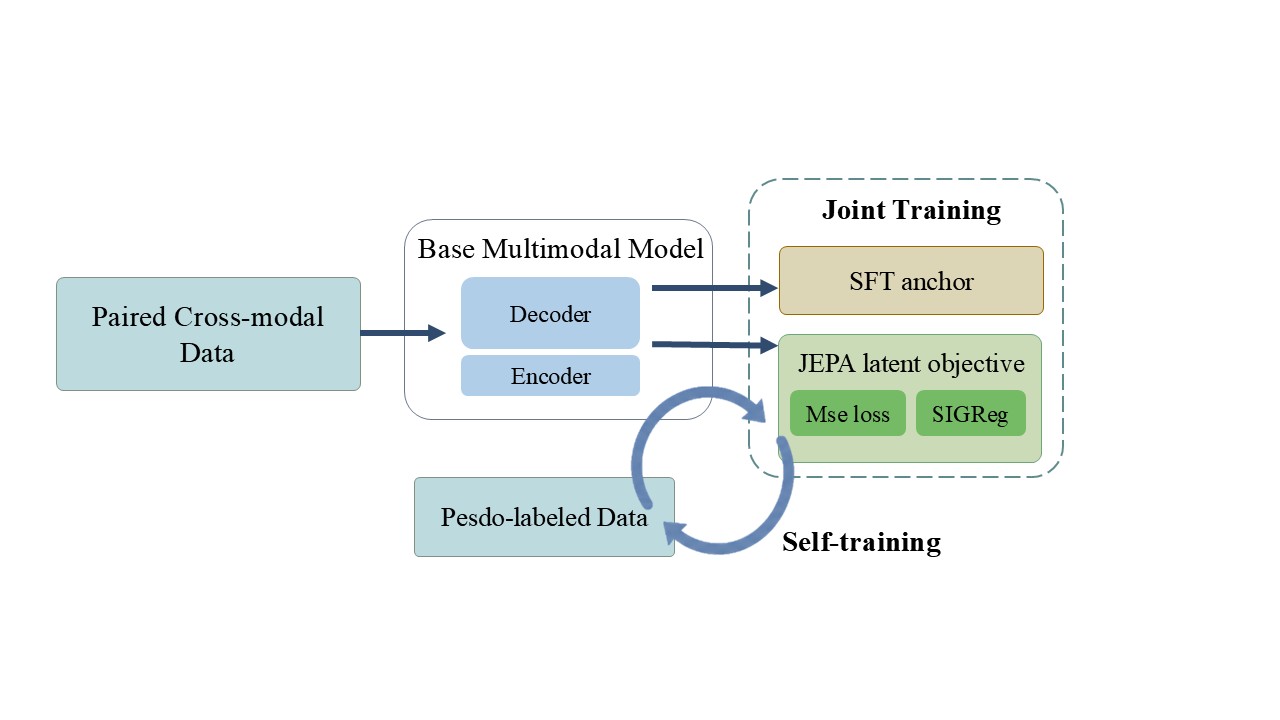}
\caption{
Macro view of \method{}.
Paired image and SMILES views are processed by one image-to-SMILES VLM, while SFT anchors causal generation and the JEPA objective regularizes intermediate hidden states through latent MSE and SIGReg.
The latent self-training loop is shown only as an abstract training extension; data-source and curriculum details are separated from this core mechanism view.
}
\label{fig:macro-flow}
\end{figure}

Autoregressive SFT learns exact translation from reliable pairs, while the latent branch extracts predictive constraints from paired data, including synthetic examples already used during pretraining.
The latent branch masks coherent image regions and symbolic spans, then predicts the corresponding intermediate states from the remaining multimodal context so that evidence from either modality can support the other.
Both branches update the same backbone, allowing latent MSE to shape the representations used for causal decoding while modality-wise SIGReg prevents collapse~\citep{balestriero2025lejepa}.
In the gold-pair setting, this provides a distinct use for the synthetic pretraining corpus after real-image adaptation: its pairs supervise latent recovery while Real-SFT continues to supervise output tokens.

\method{} further provides a noise-tolerant route for self-training with imperfect pseudo-labeled pairs.
Instead of promoting model-generated tokens to exact targets, it learns cross-modal structure that remains predictable from the paired image and sequence, thereby refining the shared representation without converting local pseudo-label errors into categorical supervision.

We evaluate \method{} on OCSR as a stringent instance of continuous-to-symbolic translation.
Under a shared exact-match protocol, clean-pair training improves both WildMol Simple and Markush over the SFT anchor.
The SFT anchor and gold-pair AC-JEPA share complete MolParser-7M synthetic pretraining followed by Real-SFT adaptation, so their comparison concerns reuse of existing pairs through latent prediction rather than access to additional unique data.
Using the same model-generated pairs, latent self-training provides further gains, whereas direct pseudo-label SFT degrades.
These results support SFT-anchored latent prediction as an effective training strategy for exact cross-modal translation with both curated and imperfect pairs.

Our contributions are:
\begin{enumerate}[leftmargin=*]
  \item We introduce \method{}, an SFT-compatible framework integrating cross-modal latent prediction with causal decoding for exact visual-to-symbolic translation.
  \item We develop a bidirectional cross-modal latent-learning mechanism that captures structure predictable from multimodal context while maintaining non-degenerate representations within each modality.
  \item We derive local bounds that relate decoder-relevant latent residuals to next-token decision stability and characterize clean-output stability under pseudo-label perturbations.
  \item We validate \method{} on OCSR, showing gains over SFT with clean pairs and over direct pseudo-label SFT during self-training.
\end{enumerate}

\section{Related Work}
\label{sec:related-work}
\paragraph{Multimodal representation learning.}
Multimodal representation learning seeks to capture content shared across paired modalities while retaining the fine-grained information required for modality-specific prediction.
Contrastive methods such as CLIP and ALIGN align global image--text representations and are effective at learning factors shared across paired observations~\citep{radford2021clip,jia2021align}.
Fusion and generative models such as ALBEF and BLIP combine alignment with cross-modal interaction or language generation, while BEiT-3 extends masked modeling across images, text, and paired inputs~\citep{li2021albef,li2022blip,wang2022beit3}.
These objectives support transferable multimodal representations and capable generators, but exact scientific translation can additionally depend on modality-specific local structure that is not captured by global correspondence alone.
\method{} targets this setting by applying local masked-state prediction to modality-indexed positions within the same fused representation used for autoregressive decoding.
SFT preserves the task-specific output interface, while the latent objective requires the intermediate states to retain cross-modal structure recoverable from the remaining context.

\paragraph{Latent predictive learning and multimodal JEPA.}
Latent predictive learning moves the prediction target from the observation space to a learned representation space.
Whereas masked autoencoding reconstructs missing observations, data2vec predicts contextualized latent targets across speech, vision, and language, and I-JEPA predicts representations of image regions from visible context~\citep{he2022mae,baevski2022data2vec,assran2023ijepa}.
The abstraction induced by such objectives depends on what is masked, which context remains visible, and how the target representation is constructed.
V-JEPA 2 extends latent prediction to video, while LeJEPA couples predictability with SIGReg-based control of representation collapse~\citep{assran2025vjepa2,balestriero2025lejepa}.
Recent cross-modal variants impose different predictive relations: CrossJEPA transfers from a frozen image target to a point-cloud encoder, VL-JEPA predicts text embeddings from visual input and a query, and MJEPA performs intra- and cross-modal prediction with a unified audio-visual encoder~\citep{perera2025crossjepa,chen2025vljepa,teotia2026mjepa}.
Their primary output is an encoder representation or embedding prediction.
In contrast, \method{} places bidirectional masked-state prediction inside the same autoregressive image-to-SMILES generator that is trained and deployed for exact symbolic translation.
The central compatibility constraint is therefore not latent prediction alone, but preserving an intermediate state that remains useful to the upper causal decoder.

\paragraph{Self-training with imperfect pseudo-labels.}
Self-training expands supervised data by using a trained model to generate targets for unlabeled examples~\citep{lee2013pseudolabel}.
Modern methods reduce confirmation errors through confidence filtering, consistency regularization, augmentation, teacher--student updates, or momentum distillation~\citep{sohn2020fixmatch,xie2020noisystudent,li2021albef}.
Their effectiveness nevertheless depends on the structure of the teacher's errors, since recurring mistakes can be reinforced when predictions are reused as hard supervision~\citep{song2020noisylabels}.
This issue is particularly consequential for symbolic generation: a sequence may be syntactically plausible and locally close to the correct output while still representing a different object.
Direct pseudo-label SFT treats such a token substitution as an exact categorical target.
\method{} instead uses the pseudo-sequence as part of a masked multimodal view and learns by predicting its continuous intermediate states.
Our experiments compare these two training routes using the same pseudo-labeled data rather than attributing their difference to data selection.

\paragraph{Optical chemical structure recognition.}
\ocsr{} converts molecular diagrams from papers, patents, and scanned documents into machine-readable molecular graphs or line notations such as SMILES~\citep{weininger1988smiles}.
Unlike ordinary text recognition, it must recover graph topology, stereochemistry, abbreviations, drawing conventions, and, in the Markush setting, variable molecular structures.
Recent neural systems approach this problem through direct image-to-string generation, chemistry-constrained graph prediction, graph classification, or large-scale end-to-end recognition~\citep{campos2021img2smi,qian2023molscribe,morin2023molgrapher,fang2024molparser}.
MolSight and MolRecBench-Wild further expose the importance of noisy pretraining and generalization to real scientific documents~\citep{zhang2025molsight,yang2026molrecbench}.
Most progress in OCSR has focused on architectures, output representations, data scale, and decoding constraints.
Our work studies a complementary training question: how an SFT-capable generator can learn additional cross-modal structure from clean and pseudo-labeled pairs through an auxiliary latent predictive objective.

\section{Method}
\label{sec:method}
\method{} integrates autoregressive SFT and latent prediction within one image-to-SMILES VLM.
SFT trains the model through autoregressive next-token prediction, whereas JEPA predicts latent states at an intermediate decoder layer.
Both losses update the shared backbone, so latent prediction shapes the representation used by the generator rather than a separate pretrained encoder.
A bidirectional predictor recovers masked states using the remaining multimodal context, and modality-wise SIGReg prevents latent collapse.
The same latent objective subsequently incorporates pseudo-labeled pairs without treating their generated SMILES as default token-level targets.
\Cref{fig:model-structure} shows the resulting architecture and training losses.
\Cref{subsec:sft-anchor,subsec:jepa-objective,subsec:curriculum,subsec:conditional-analysis} describe the latent interface, predictive objective, training curriculum, and analysis, respectively.

\begin{figure}[h]
\centering
\includegraphics[width=\linewidth]{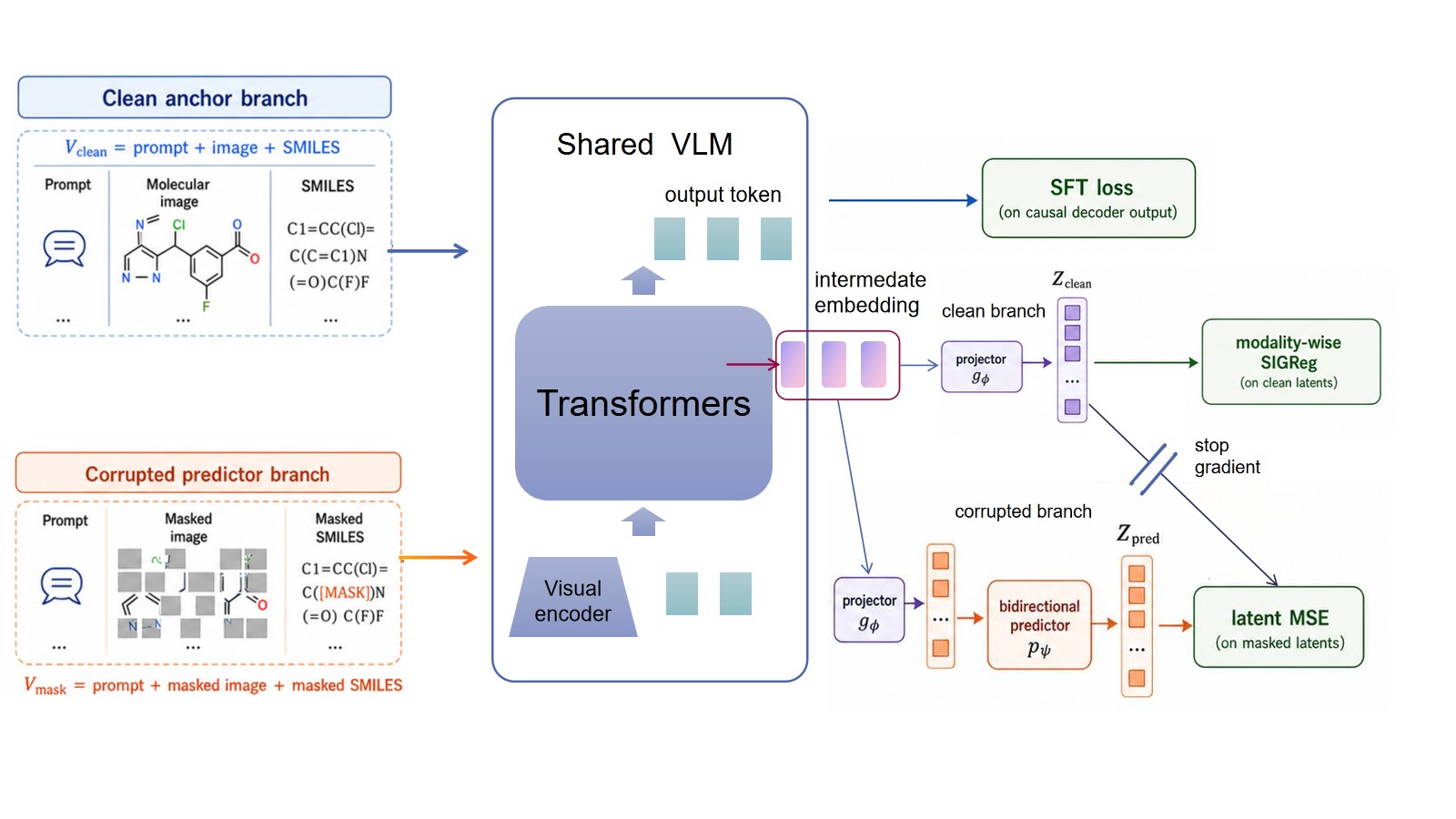}
\caption{
Model structure and training losses of \method{}.
SFT supervises the causal decoder, while JEPA predicts intermediate latents under structured masking.
The clean latent is detached only as the MSE target, while SIGReg regularizes the undetached clean image and SMILES latents separately.
}
\label{fig:model-structure}
\end{figure}

\subsection{SFT-Anchored Backbone and Latent Interface}
\label{subsec:sft-anchor}

The position of the latent interface determines whether predictive learning can affect both cross-modal fusion and downstream decoding.
We attach it to an intermediate decoder layer after visual and symbolic context have been fused but before next-token logits are produced.
The JEPA path projects this state for latent prediction, whereas the SFT path continues through the remaining causal decoder and language-model head.
Both losses update the shared model below the interface, while SFT directly supervises the upper decoding path.
This arrangement keeps the latent objective inside the task-capable generator while preserving the causal output path used at inference.

Let $I$ be a molecular image, $S=(S_1,\ldots,S_T)$ its gold SMILES sequence, and $\pi$ the task prompt.
The SFT anchor trains the backbone $f_\theta$ autoregressively,
\begin{equation}
  \mathcal{L}_{\mathrm{SFT}}
  =
  -\sum_{t=1}^{T}
  \log p_\theta(S_t \mid \pi, I, S_{<t}) .
  \label{eq:sft-loss}
\end{equation}
This objective ties the shared representation to the grammar and causal decoding procedure used at inference.

For a paired symbolic view $S_{\mathrm{view}}$, we construct clean and masked multimodal views
\begin{equation}
  \begin{aligned}
  \mathbf{V}_{\mathrm{clean}}
  &=(\pi,I,S_{\mathrm{view}}),\\
  \mathbf{V}_{\mathrm{mask}}
  &=
  m(\mathbf{V}_{\mathrm{clean}};
  \mathcal{M}_{\mathrm{img}},\mathcal{M}_{\mathrm{smi}}).
  \end{aligned}
  \label{eq:multimodal-views}
\end{equation}
where $\mathcal{M}_{\mathrm{img}}$ and $\mathcal{M}_{\mathrm{smi}}$ are the masked image and SMILES positions.
Let $H_\theta^\ell(\mathbf{V})$ be the hidden sequence at decoder layer $\ell$, $g_\phi$ a projector, and $p_\psi$ the predictor.
The projector maps the fused hidden states into a compact space in which the masked-view representation is trained to match its clean counterpart.
The clean latent, detached target, and masked-view prediction are
\begin{equation}
  \begin{aligned}
  Z_{\mathrm{clean}}&=g_\phi(H_\theta^\ell(\mathbf{V}_{\mathrm{clean}})),&
  Z_{\mathrm{tgt}}&=\operatorname{sg}(Z_{\mathrm{clean}}),\\
  Z_{\mathrm{pred}}&=
  p_\psi(g_\phi(H_\theta^\ell(\mathbf{V}_{\mathrm{mask}}))),
  \end{aligned}
  \label{eq:latent-prediction}
\end{equation}
where $\operatorname{sg}$ blocks predictive-MSE gradients through the clean target.
The clean branch therefore supplies a fixed target for each prediction step, while both the masked branch and the shared backbone below the interface remain trainable.
Latent recovery is optimized only at masked positions and separately for each modality,
\begin{equation}
  \begin{aligned}
  \mathcal{L}_{\mathrm{MSE}}
  &=
  \sum_{m\in\{\mathrm{img},\mathrm{smi}\}}
  \frac{\lambda_{\mathrm{mse}}^{m}}{|\mathcal{M}_m|}
  \sum_{i\in\mathcal{M}_m}\\[-2pt]
  &\qquad
  \left\|
    Z_{\mathrm{pred},i}-Z_{\mathrm{tgt},i}
  \right\|_2^2 .
  \end{aligned}
  \label{eq:jepa-mse}
\end{equation}
The continuous loss complements rather than replaces token-level SFT.
SFT anchors exact generation and the output grammar, while masked latent prediction directly constrains the shared intermediate representation consumed by the decoder.

\subsection{Bidirectional Prediction and Latent Regulation}
\label{subsec:jepa-objective}

AC-JEPA trains its intermediate representation by recovering masked image and symbolic states from the remaining multimodal context, while modality-wise regulation keeps the predictive target space non-degenerate.
We construct the recovery task by replacing contiguous SMILES spans and rectangular image-patch blocks with modality-specific mask tokens; prompt, padding, and special tokens remain unmasked.
Masking coherent regions removes local evidence and makes recovery rely on the surrounding structure available within and across modalities.
Bidirectional attention lets the predictor recover masked states from the full corrupted multimodal context, while the SFT generation path remains causal.
\Cref{app:structured-corruption} specifies the structured masking procedure.

Because predictive objectives can admit collapsed representations, we apply the Sketched Isotropic-Gaussian Regularizer (SIGReg) to the clean latent distribution~\citep{balestriero2025lejepa}.
For modality $m$, let $\mathcal{I}_m$ contain its eligible clean-view positions in a mini-batch and $\mathcal{Z}_m=\{Z_{\mathrm{clean},i}:i\in\mathcal{I}_m\}$.
We regularize image and SMILES states separately,
\begin{equation}
  \mathcal{L}_{\mathrm{SIGReg}}
  =
  \sum_{m\in\{\mathrm{img},\mathrm{smi}\}}
  \lambda_{\mathrm{sig}}^{m}
  \mathrm{SIGReg}(\mathcal{Z}_m).
  \label{eq:sigreg-total}
\end{equation}
This separation prevents one modality's variance from hiding collapse in the other.
SIGReg acts on the undetached clean latents, whereas stop-gradient applies only to the MSE target.
Accordingly, stop-gradient stabilizes each prediction target, whereas SIGReg constrains the representation from which the targets are formed.
The complete latent objective is
$\mathcal{L}_{\mathrm{JEPA}}=
\mathcal{L}_{\mathrm{MSE}}+\mathcal{L}_{\mathrm{SIGReg}}$.

\subsection{From SFT Anchoring to Latent Self-Training}
\label{subsec:curriculum}

The training curriculum first establishes exact translation on reliable pairs and then uses the latent objective to incorporate open-environment images without gold sequences.
This ordering is important because the JEPA branch should shape an already task-capable representation rather than define the symbolic task from scratch.
\Cref{fig:training-workflow} summarizes the four training stages.
Base SFT and all AC-JEPA variants share domain continued pretraining (CPT) on the complete synthetic MolParser-7M image--SMILES corpus, followed by autoregressive fine-tuning on Real-SFT 100K~\citep{fang2024molparser}.
Synthetic pretraining establishes broad task competence, while real-image fine-tuning adapts the generator to document images.
We freeze this Base SFT backbone to warm up the new JEPA modules; gold-pair joint training and each latent self-training run then initialize independently from the resulting checkpoint.

\begin{figure}[h]
\centering
\includegraphics[width=\linewidth]{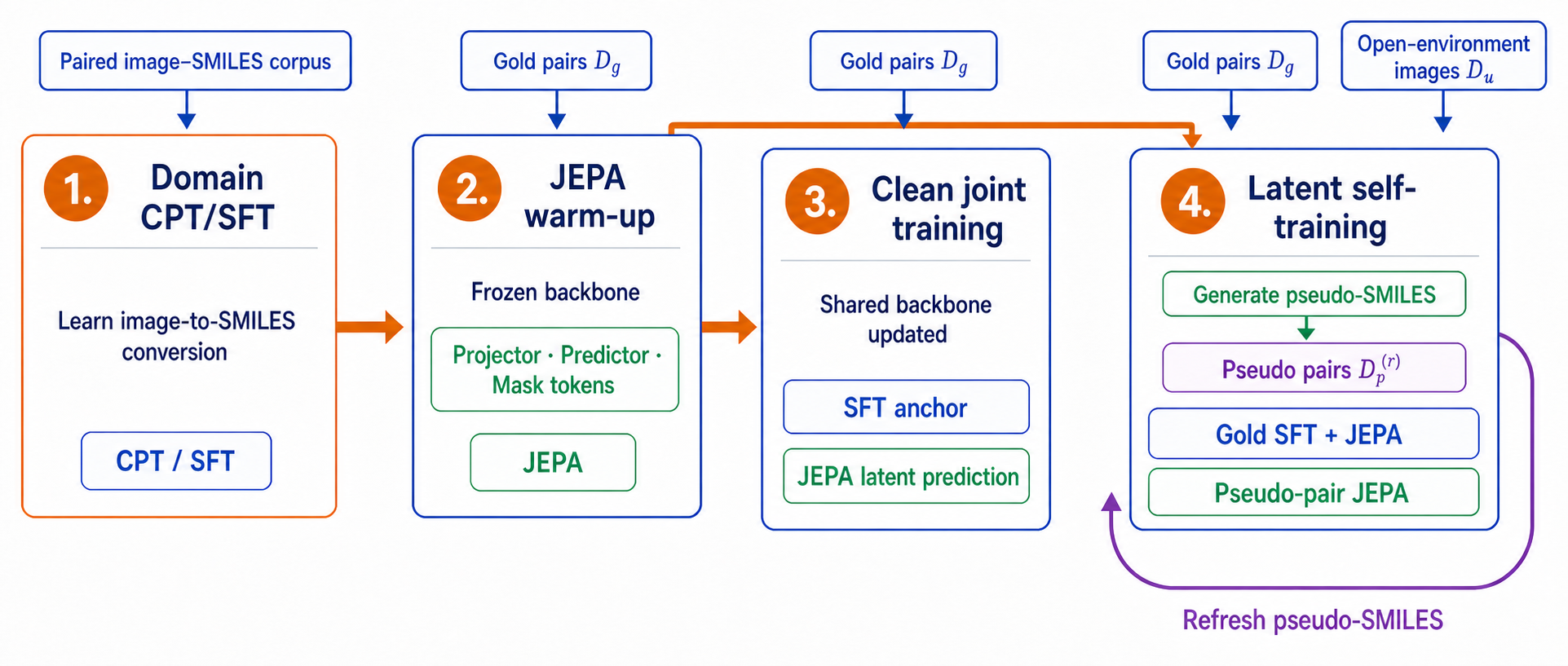}
\caption{
Training curriculum.
Base SFT first learns from the complete synthetic MolParser-7M corpus and is then fine-tuned on Real-SFT 100K.
After JEPA-module warm-up with that backbone frozen, gold-pair joint training and latent self-training start independently from the shared checkpoint.
Gold-pair joint training reuses the already-seen 7M pairs for latent prediction while retaining Real-SFT for token supervision.
Self-training adds open-environment images through model-generated pseudo-SMILES while gold-pair SFT remains active as the anchor; pseudo-label sets are refreshed between rounds.
}
\label{fig:training-workflow}
\end{figure}

Let $\mathcal{D}_{\mathrm{sft}}$ contain the reliable pairs assigned token-level supervision and $\mathcal{D}_{\mathrm{jepa}}^{g}$ contain the gold pairs used for latent prediction.
The anchored objective is
\begin{equation}
  \begin{aligned}
  \mathcal{L}_{\mathrm{anchor}}
  &=
  \lambda_{\mathrm{sft}}
  \mathbb{E}_{(I,S)\sim\mathcal{D}_{\mathrm{sft}}}
  [\mathcal{L}_{\mathrm{SFT}}(I,S)]\\
  &\quad+
  \mathbb{E}_{(I,S)\sim\mathcal{D}_{\mathrm{jepa}}^{g}}
  [\mathcal{L}_{\mathrm{JEPA}}(I,S)] .
  \end{aligned}
  \label{eq:anchor-objective}
\end{equation}
The two datasets may overlap but need not coincide: joint training means that both losses update the shared backbone, not that they are evaluated on the same examples.
In our gold-pair setting, $\mathcal D_{\mathrm{sft}}$ is Real-SFT and $\mathcal D_{\mathrm{jepa}}^g$ reuses the synthetic MolParser-7M pairs already used during CPT; no unique paired examples are added to the Base SFT training history.
Our small-scale mixed-SFT observation in \Cref{subsec:main-results} motivates retaining real-image token supervision while reusing synthetic pairs through latent prediction.
At refresh round $r$, the current model generates $\hat S_i^{(r)}$ for each $I_i\in\mathcal D_u$, forming
$\mathcal D_p^{(r)}=\{(I_i,\hat S_i^{(r)})\}$.
\method{} then treats each pseudo-SMILES as the symbolic view $S_{\mathrm{view}}$ of a masked latent prediction problem rather than as default token-level supervision:
\begin{equation}
  \mathcal{L}^{(r)}
  =
  \mathcal{L}_{\mathrm{anchor}}
  +
  \eta_r
  \mathbb{E}_{(I,\hat{S})\sim\mathcal{D}_p^{(r)}}
  [
    q(I,\hat{S})\,\mathcal{L}_{\mathrm{JEPA}}(I,\hat{S})
  ],
  \label{eq:latent-self-training-objective}
\end{equation}
where $q(I,\hat S)$ is a selection or reliability weight and $\eta_r$ controls the pseudo-labeled stream.
The supervised stream continues to supply the token-level anchor throughout this stage.
Direct pseudo-label SFT is retained as a comparison, so the two routes differ in how the same pseudo-labeled data provide supervision.
\Cref{alg:ac-jepa-training} summarizes the complete training procedure.

\subsection{Local Analysis of Task-Aligned Representations and Noise Tolerance}
\label{subsec:conditional-analysis}

To obtain a tractable account of the proposed mechanism, we analyze a local model of AC-JEPA under simplifying assumptions on decoder margins, subspace overlap, anchored curvature, and pseudo-label influence.
The resulting bounds identify conditions under which lower decoder-relevant latent residuals tighten the next-token decision-reversal bound and pseudo-label perturbations remain within the clean-output margin.
They provide a mechanism-level interpretation rather than a general guarantee for the full nonlinear training dynamics.
Detailed proofs and further interpretation are provided in \Cref{app:conditional-analysis-proofs}.

At masked SMILES position $i$, let $H_i$ be the clean intermediate state, $C_i$ the corrupted multimodal context, and
$\Delta_{i,s^-}(H_i)\approx d_{i,s^-}^{\top}H_i+c_{i,s^-}$
the logit gap between the correct next token $s^+$ and competitor $s^-$.
Let $\mu_{i,s^-}(C_i)>0$ and $v_{i,s^-}(C_i)$ be its conditional mean and variance, representing the average correct-token margin and residual variation.
With context-conditioned state residual $U_i=H_i-\mathbb E[H_i\mid C_i,S_{i+1}=s^+]$ and latent-interface Jacobian $J_i$, define the decoder-relevant residual energy
$R_{\mathrm J,i}(C_i)=\mathbb E[\|J_iU_i\|_2^2\mid C_i,S_{i+1}=s^+]$.
We assume the decoder margin direction is controlled through this interface:
$d_{i,s^-}d_{i,s^-}^{\top}\preceq\kappa J_i^\top J_i$.

\begin{theorem}[SFT readout stability under latent residuals]
\label{thm:task-anchored-readout}
Let
$P_{i,s^-}=\Pr[\Delta_{i,s^-}(H_i)\leq0\mid C_i,S_{i+1}=s^+]$.
Under the local model above,
\begin{equation}
  \begin{aligned}
  P_{i,s^-}
  &\leq
  \frac{v_{i,s^-}(C_i)}
       {v_{i,s^-}(C_i)+\mu_{i,s^-}(C_i)^2}\\
  &\leq
  \frac{\kappa R_{\mathrm J,i}(C_i)}
       {\kappa R_{\mathrm J,i}(C_i)+\mu_{i,s^-}(C_i)^2}.
  \end{aligned}
  \label{eq:conditional-readout-bound}
\end{equation}
Thus, at a fixed positive mean margin, reducing the decoder-relevant latent residual tightens the decision-reversal bound.
\end{theorem}

The symbolic MSE decomposes as
$\mathcal L_{\mathrm{MSE}}^{\mathrm{smi}}
=\overline R_{\mathrm J}^{\mathrm{smi}}+\mathcal E_{\mathrm{pred}}^{\mathrm{smi}}$.
After predictor warm-up, reducing MSE beyond the initial excess risk lowers the Bayes latent residual and tightens the averaged bound under decoder--JEPA overlap and SFT margin preservation.

To quantify pseudo-label effects on clean outputs, consider $n$ pseudo constraints and take as the clean reference a stationary point $\theta^\star$ of the counterfactual anchored objective with their corresponding true sequences $\{S_j\}_{j=1}^{n}$.
For pseudo constraint $j$, define the latent residual $r_j$ and its perturbation $e_j$ after replacing $S_j$ with $\widehat S_j$:
\begin{equation}
  \begin{aligned}
  r_j(\theta;S_j)
  &=
  Z_{\mathrm{pred},j}(\theta;S_j)
  -
  Z_{\mathrm{tgt},j}(\theta;S_j),\\
  e_j
  &=
  r_j(\theta^\star;\widehat S_j)
  -
  r_j(\theta^\star;S_j).
  \end{aligned}
  \label{eq:pseudo-latent-residual-perturbation}
\end{equation}
Let $A_j=\partial r_j(\theta;S_j)/\partial\theta|_{\theta^\star}$ be the clean-residual Jacobian, $K\succ0$ the local anchored curvature, and $\lambda$ the effective pseudo-MSE weight.
Under the local shared-Jacobian approximation, with the Jacobian-change and SIGReg remainders made explicit in \Cref{app:pseudolabel-output-influence}, the first-order parameter shift is
\begin{equation}
  u_{\mathrm{pseudo}}
  \approx
  -\frac{\lambda}{n}K^{-1}
  \sum_{j=1}^{n}A_j^\top e_j .
  \label{eq:pseudo-parameter-influence}
\end{equation}
For clean context $x$, let $m_{x,yk}>0$ be the margin between correct token $y$ and competitor $k$, $d_{x,yk}$ its state gradient, and $J_x$ the parameter-to-state Jacobian.
Pseudo constraint $j$ has signed margin influence
\begin{equation}
  \zeta_{x,yk,j}
  =
  d_{x,yk}^{\top}J_xK^{-1}A_j^\top e_j .
  \label{eq:pseudo-output-influence}
\end{equation}

\begin{theorem}[Clean-output stability under pseudo-label perturbations]
\label{thm:anchored-latent-self-training}
Suppose the influences are independent with
$\mathbb E[\zeta_{x,yk,j}]=\beta_{x,yk}$ and
$\operatorname{Var}(\zeta_{x,yk,j})\leq\sigma_{x,yk}^2$.
The first-order margin is
$G_{x,yk}^{\mathrm{lin}}=m_{x,yk}-\lambda n^{-1}\sum_j\zeta_{x,yk,j}$.
If $m_{x,yk}>\lambda\beta_{x,yk}$, then
\begin{equation}
  \Pr[G_{x,yk}^{\mathrm{lin}}\leq0]
  \leq
  \frac{\lambda^2\sigma_{x,yk}^2/n}
       {\lambda^2\sigma_{x,yk}^2/n+
        (m_{x,yk}-\lambda\beta_{x,yk})^2}.
  \label{eq:pseudo-output-stability}
\end{equation}
\end{theorem}

The bound separates mean signed influence $\beta_{x,yk}$, whose positive values consume clean margin, from variance that decreases across independent pseudo constraints.
Conditional centering $\mathbb E[e_j\mid A_j]=0$ is sufficient for zero mean first-order influence under this local model.

\section{Experiments}
\label{sec:experiments}
\subsection{Experimental Setup}
\label{subsec:setup}
All experiments use Qwen3-VL-2B~\citep{bai2025qwen3vl} as the backbone and draw paired training data from UniParser/MolParser-7M~\citep{fang2024molparser}.
Base SFT and all AC-JEPA variants share the same two-stage training history: autoregressive domain pretraining on the complete synthetic MolParser-7M corpus, followed by fine-tuning on Real-SFT 100K.
We use ``MolParser-7M'' below for the synthetic pretraining corpus and ``Real-SFT'' for the real-image adaptation set.
Both data sources have therefore already contributed to Base SFT; the gold-pair comparison introduces no additional unique training pairs.
Following \Cref{subsec:curriculum}, we warm up the newly initialized JEPA modules with this backbone frozen and then jointly optimize token-level SFT on Real-SFT and latent prediction on the previously seen synthetic pairs.
Latent self-training subsequently extends the JEPA stream to pseudo-labeled open-environment images, while direct SFT on the same pseudo-labels serves as the token-level self-training baseline.
For rounds R1 and R2, the unlabeled pool includes molecular images from USPTO-Mol and structure figures extracted from scientific papers.
We select checkpoints on the MolParser-7M validation split and evaluate them on the WildMol Simple and Markush sets, each containing 10,000 images collected from scientific papers and patents~\citep{fang2024molparser}.
Under a shared decoding and evaluation protocol, we report canonicalized exact-match accuracy using an RDKit-based evaluator customized for extended SMILES (E-SMILES).

\subsection{Results and Training Dynamics Across the Curriculum}
\label{subsec:main-results}
We first evaluate how performance evolves across the AC-JEPA training curriculum, from the Base SFT model to gold-pair joint training and two rounds of latent self-training.
After warming up the newly initialized JEPA modules with the backbone frozen, all three AC-JEPA configurations start from the resulting checkpoint and are optimized independently.
Gold-pair AC-JEPA retains Real-SFT as the token-level anchor and reuses the previously seen synthetic MolParser-7M pairs through JEPA.
For R1 and R2, the unlabeled pool includes molecular images from USPTO-Mol and structure figures extracted from scientific papers.
Each round generates pseudo-SMILES for this pool and adds the resulting pairs to the JEPA stream, while reliable-pair SFT remains the task anchor.

\begin{table}[h]
  \centering
  \caption{Accuracy (\%) across the AC-JEPA curriculum on WildMol Simple and Markush. Every row in our curriculum inherits the same MolParser-7M pretraining and Real-SFT fine-tuning; gold-pair AC-JEPA reuses existing pairs through JEPA. MolParser-Base is a published reference from \citet{fang2024molparser}, with a different architecture and optimization recipe.}
  \label{tab:main-results}
  \begin{tabular}{lcc}
    \toprule
    Method & Simple Acc. (\%) & Markush Acc. (\%) \\
    \midrule
    \multicolumn{3}{l}{\emph{Published reference}} \\
    MolParser-Base & 76.90 & 38.10 \\
    \addlinespace
    \multicolumn{3}{l}{\emph{Our curriculum}} \\
    Base SFT & 67.40 & 36.48 \\
    Gold-pair \method{} & 76.07 & 42.28 \\
    \method{} self-training, round 1 & 78.51 & 42.78 \\
    \method{} self-training, round 2 & \textbf{80.23} & \textbf{48.28} \\
    \bottomrule
  \end{tabular}
\end{table}

\Cref{tab:main-results} summarizes performance across the proposed training curriculum.
Starting from Base SFT, joint gold-pair training with \method{} raises accuracy by 8.67 points on Simple and 5.80 points on Markush.
This configuration reaches comparable Simple accuracy to MolParser-Base and improves Markush by 4.18 points.

Self-training R1 and R2 extend latent prediction with these pseudo-pairs while retaining the SFT anchor on clean data.
R2 reaches 80.23\% on Simple and 48.28\% on Markush, exceeding the independently initialized gold-pair configuration by 4.16 and 6.00 points, respectively.
The higher endpoint accuracies of both self-training runs show that the pseudo-pair streams benefit AC-JEPA in this OCSR setting.

\begin{figure*}[t]
\centering
\includegraphics[width=\textwidth]{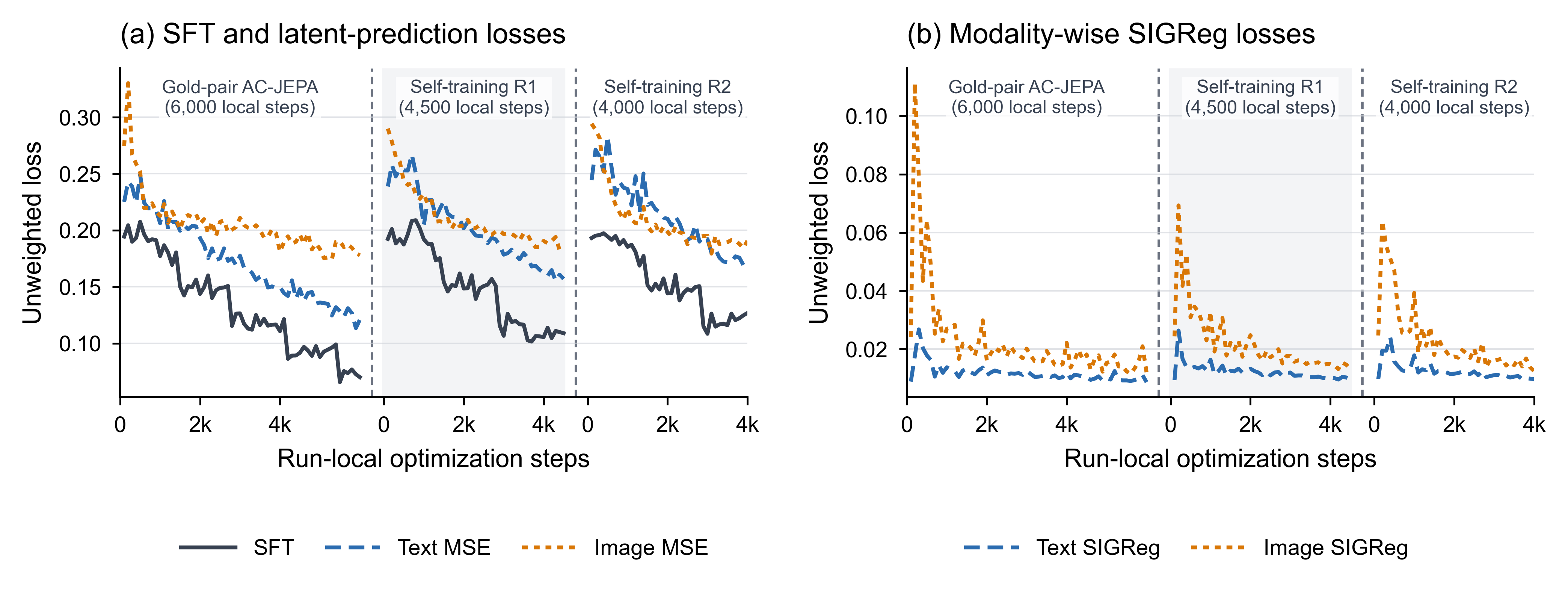}
\caption{
Component losses from three independent AC-JEPA runs initialized from the same CPT-trained, JEPA-warmed checkpoint.
Curves show every 100 local optimization steps for gold-pair training (6,000 steps), self-training R1 (4,500 steps), and self-training R2 (4,000 steps).
The blocks are juxtaposed for compact comparison and are not connected across runs.
All curves report unweighted loss components.
}
\label{fig:multistage-training-losses}
\end{figure*}

\Cref{fig:multistage-training-losses} reports component losses for the three independent runs.
All start from the same CPT-trained, JEPA-warmed checkpoint with separate optimization schedules, and the juxtaposed blocks show run-local rather than sequential steps.
Within each run, SFT and both modality-specific MSE losses decrease; image SIGReg also decreases, while text SIGReg remains low and bounded.
These curves show joint optimization of the token-level and latent objectives; \Cref{tab:main-results} reports endpoint accuracy.

\paragraph{Reusing synthetic pairs after real-image adaptation.}
The shared pretraining history is important for interpreting the gold-pair gain: MolParser-7M has already contributed to Base SFT, and AC-JEPA changes how its synthetic pairs are reused after Real-SFT adaptation.
In a small-scale implementation, we also used synthetic pairs from MolParser-7M alongside Real-SFT for token-level SFT.
This mixture did not improve over Real-SFT-only adaptation and was more prone to overfitting.
In contrast, retaining Real-SFT token supervision and reusing the synthetic pairs through AC-JEPA yields the improvements in \Cref{tab:main-results}.
The pilot therefore supports the importance of the supervision route rather than an explanation based simply on adding paired examples or extending mixed token-level training.
We report the pilot as qualitative design evidence, with its reporting scope specified in \Cref{app:synthetic-reuse-pilot}.

\subsection{Pseudo-Label Noise Analysis}
\label{subsec:noise}

\paragraph{Model-generated pseudo-label noise.}
We first compare two ways of using pseudo-labels generated for unlabeled open-environment images: direct token-level SFT and latent self-training with \method{}.
Both settings use model-generated pseudo-labels, which contain naturally occurring recognition and SMILES errors rather than manually specified corruptions.
\begin{table}[h]
  \centering
  \caption{Accuracy (\%) when model-generated pseudo-labels are used for direct token-level SFT or AC-JEPA latent self-training.}
  \label{tab:model-pseudolabel-results}
  \begin{tabular}{lcc}
    \toprule
    Pseudo-label route & Simple & Markush \\
    \midrule
    Direct pseudo-label SFT & 63.47 & 36.04 \\
    \method{} self-training, round 1 & 78.51 & 42.78 \\
    \method{} self-training, round 2 & \textbf{80.23} & \textbf{48.28} \\
    \bottomrule
  \end{tabular}
\end{table}

Relative to direct pseudo-label SFT, the first AC-JEPA self-training round improves Simple and Markush accuracy by 15.04 and 6.74 points, respectively.
The second round increases these margins to 16.76 and 12.24 points.
These results show a clear empirical advantage for routing pseudo-pairs through latent prediction in the current training pipeline.
Because the recorded runs do not yet match total training budget exactly, we treat this comparison as evidence from the realized curriculum rather than as an isolated causal comparison between loss functions.

\paragraph{Controlled synthetic SMILES noise.}
To compare the noise tolerance of direct SFT and AC-JEPA independently of the unknown error distribution of model-generated pseudo-labels, we construct a compact controlled testbed from the clean MolParser Real-SFT subset~\citep{fang2024molparser} and MolGrapher-Synthetic-300K~\citep{morin2023molgrapher}.
We corrupt only the SMILES labels of a controlled fraction $\rho_{\mathrm{seq}}$ of the MolGrapher examples; their paired molecular images and all Real-SFT examples remain unchanged.
For each selected MolGrapher label, the current corruption procedure perturbs 5--15\% of its SMILES characters.
In AC-JEPA training, supervised generation is anchored by the clean Real-SFT data, whereas the corrupted MolGrapher pairs are used only by the JEPA latent-prediction objective.
In the direct-SFT control, both the clean Real-SFT pairs and the corrupted MolGrapher pairs are optimized with the token-level SFT objective.
We compare these routes across increasing sequence-level noise and evaluate both on the unchanged WildMol Simple and Markush benchmarks in \Cref{tab:controlled-noise-results}.
As $\rho_{\mathrm{seq}}$ rises to 80\%, direct SFT degrades sharply, whereas \method{} retains gains over the clean Base SFT on both benchmarks.
\begin{table}[h]
  \centering
  \caption{Accuracy (\%) under controlled SMILES corruption. $\rho_{\mathrm{seq}}$ denotes the fraction of corrupted MolGrapher labels; 5--15\% of characters are perturbed in each selected label.}
  \label{tab:controlled-noise-results}
  \begin{tabular}{lccc}
    \toprule
    Objective & $\rho_{\mathrm{seq}}$ & Simple & Markush \\
    \midrule
    Direct SFT & 20\% & 67.08 & 21.85 \\
    Direct SFT & 50\% & 66.03 & 21.62 \\
    Direct SFT & 80\% & 51.37 & 4.45 \\
    \addlinespace
    \method{} joint training & 20\% & \textbf{73.67} & \textbf{45.98} \\
    \method{} joint training & 50\% & \textbf{73.26} & \textbf{45.45} \\
    \method{} joint training & 80\% & \textbf{73.07} & \textbf{43.01} \\
    \bottomrule
  \end{tabular}
\end{table}

\subsection{Ablation Studies}
\label{subsec:ablations}

\paragraph{Intermediate latent interface.}
The JEPA objective reads an intermediate decoder layer rather than the final generation state.
The current coarse sweep in \Cref{tab:layer-sweep} favors layer 17 among the four recorded choices, suggesting that states too close to the final decoder output are less effective as predictive targets.
\begin{table}[h]
  \centering
  \caption{Preliminary Simple accuracy (\%) for the recorded decoder-layer choices used as the JEPA latent interface.}
  \label{tab:layer-sweep}
  \begin{tabular}{lc}
    \toprule
    Latent-interface layer & Simple Acc. (\%) \\
    \midrule
    12 & 76.07 \\
    17 & \textbf{77.67} \\
    23 & 75.61 \\
    26 & 71.31 \\
    \bottomrule
  \end{tabular}
\end{table}

\paragraph{Masking configuration.}
\Cref{tab:mask-ratio-sweep} varies the masking ratio and masked modalities under the same structured image-block and SMILES-span corruption.
Joint image--SMILES masking at ratio 0.35 performs best, while lower or higher ratios reduce Simple accuracy.
Both SMILES-only and image-only masking remain above Base SFT, indicating that masked prediction in either modality still benefits the SFT task.
\begin{table}[h]
  \centering
  \caption{Preliminary Simple accuracy (\%) for the recorded masking-ratio and masked-modality variants.}
  \label{tab:mask-ratio-sweep}
  \begin{tabular}{lcc}
    \toprule
    Masked modalities & Ratio & Simple Acc. (\%) \\
    \midrule
    Image + SMILES & 0.20 & 75.20 \\
    Image + SMILES & 0.35 & \textbf{76.07} \\
    Image + SMILES & 0.50 & 74.60 \\
    \addlinespace
    SMILES only & 0.35 & 73.12 \\
    Image only & 0.35 & 70.54 \\
    \bottomrule
  \end{tabular}
\end{table}

\paragraph{Predictor and latent-stability diagnostics.}
Additional configuration diagnostics provide empirical support for the predictor and latent-stability choices used in the final model.
Replacing bidirectional prediction with causal attention substantially slowed predictive MSE convergence and yielded 72.95\% Simple accuracy.
Reducing the JEPA latent dimension from 1024 to 256 similarly slowed convergence and yielded 74.28\% Simple accuracy, suggesting that the JEPA latent should not be aggressively compressed within this decoder architecture.
Removing stop-gradient caused rapid latent collapse under the default SIGReg setting.
Stronger regularization preserved non-degenerate latent dispersion, but accuracy degraded rapidly.
At the default weight, pooling image and SMILES latents under one SIGReg term rapidly drove the pooled statistic toward the regularization target, and accuracy degraded substantially.
With a substantially lower pooled-SIGReg weight, predictive MSE decreased slowly, but Simple accuracy remained limited to 73.24\%.
Together, these diagnostics support bidirectional recovery, clean-target detachment, and modality-wise regulation for learning predictable intermediate states without forcing the heterogeneous modalities into one pooled statistic.

\section{Discussion and Limitations}
\label{sec:discussion}
\method{} complements rather than replaces SFT: the causal objective anchors the output language and deployable image-to-SMILES path, while latent prediction supplies an additional representation-level signal from gold or pseudo-labeled pairs.

Synthetic data serve useful but distinct roles across the curriculum.
Both Base SFT and AC-JEPA benefit from full MolParser-7M pretraining before Real-SFT adaptation; the gold-pair gain therefore does not come from access to new paired examples.
Our mixed-SFT pilot indicates that returning synthetic pairs to categorical supervision is not automatically beneficial after real-image adaptation and can increase overfitting.
Together with the gold-pair results, this observation supports AC-JEPA as a way to extract further predictive constraints from existing pairs while retaining a real-data token anchor.
The pilot remains qualitative, and its scope should be distinguished from a systematic comparison across training budgets and synthetic/real sampling ratios.

Its noise tolerance is conditional.
Continuous latent targets can soften the effect of local pseudo-label errors, but systematic bias may still shift the learned representation; our controlled-corruption results support this behavior only within the tested OCSR setting.
Broader visual-to-symbolic generalization therefore requires validation on additional tasks, and the extra pseudo-label generation and training should not be interpreted as compute-free.

\section{Conclusion}
\label{sec:conclusion}
We introduced \method{}, an SFT-compatible objective that learns visual-to-symbolic representations through bidirectional masked latent prediction while retaining causal token supervision.
On OCSR, gold-pair \method{} improves a Base SFT model that already shares complete synthetic MolParser-7M pretraining and Real-SFT adaptation, showing the usefulness of reusing existing pairs through latent prediction.
The contrasting small-scale mixed-SFT observation supports retaining real-image token supervision while assigning synthetic-pair reuse to the latent objective.
Latent self-training further improves performance with imperfect pseudo-labels, where the recorded direct pseudo-label SFT run degrades.
These results support anchored joint-embedding prediction as a practical route to more noise-tolerant cross-modal translation.

\bibliographystyle{plainnat}
\bibliography{references}

\appendix
\section{Appendix}
\label{app:appendix}
\subsection{Structured Corruption Details}
\label{app:structured-corruption}

The main text defines the image and SMILES mask sets used by the latent MSE.
For completeness, contiguous spans are sampled only within the assistant SMILES region,
\begin{equation}
  \begin{aligned}
  \mathcal{M}_{\mathrm{smi}}
  &=
  \bigcup_{b=1}^{B_{\mathrm{smi}}}
  \{s_b,\ldots,s_b+\ell_b-1\},\\
  \ell_b
  &\sim
  \mathrm{Unif}[\ell_{\min},\ell_{\max}],
  \end{aligned}
  \label{eq:app-smiles-mask}
\end{equation}
until the prescribed SMILES masking ratio is reached.
Rectangular blocks are sampled on the visual patch grid,
\begin{equation}
  \mathcal{M}_{\mathrm{img}}
  =
  \bigcup_{b=1}^{B_{\mathrm{img}}}
  \left\{(u,v):
  \begin{aligned}
  u_b&\leq u<u_b+h_b,\\[-2pt]
  v_b&\leq v<v_b+w_b
  \end{aligned}
  \right\}.
  \label{eq:app-image-mask}
\end{equation}
These masks can remove all or part of local symbolic motifs and coherent image regions without asserting that every sampled mask coincides with a complete chemical substructure.
Prompt, padding, and special tokens are excluded from both sets.
The masks are sampled before the clean and corrupted forward passes so that both branches use the same prediction positions.
If a modality has no masked positions in a micro-batch, its MSE term is skipped through a graph-valid zero.

\subsection{Target Stabilization and Gradient Routing}
\label{app:target-stabilization}

The clean branch produces the undetached latent $Z_{\mathrm{clean}}$, from which the predictive target is formed as $Z_{\mathrm{tgt}}=\operatorname{sg}(Z_{\mathrm{clean}})$.
Consequently, $\mathcal{L}_{\mathrm{MSE}}$ updates the corrupted-view backbone up to layer $\ell$, the shared projector, and the predictor, but does not move the clean target through the MSE path.
The stop-gradient does not freeze the shared backbone.
The gold-pair SFT loss still updates the complete visual encoder and causal decoder, and modality-wise SIGReg is evaluated on the undetached $Z_{\mathrm{clean}}$.
Thus, stop-gradient stabilizes the clean SFT-anchored target, whereas SIGReg supplies collapse control.

The no-detach variant allows the MSE to update both clean and corrupted representations and is retained as a target-stability diagnostic.
It is not treated as the default merely because the underlying LeJEPA regularizer can be used without a detached target.
This distinction is specific to the present shared-weight autoregressive generator, for which a latent space that is easy to predict need not remain useful for exact SMILES decoding.

\subsection{Modality-wise SIGReg Details}
\label{app:sigreg-details}

SIGReg compares one-dimensional random projections of the clean latent distribution with the characteristic function of a standard normal target~\citep{balestriero2025lejepa}.
For modality $m$, let $\mathcal{Z}_m$ contain all eligible clean-view latents in the JEPA mini-batch and let $N_m=|\mathcal{Z}_m|$.
With random unit directions
\begin{equation}
  \widetilde{a}_p^m\sim\mathcal{N}(0,I),
  \qquad
  a_p^m=\widetilde{a}_p^m/\|\widetilde{a}_p^m\|_2,
\end{equation}
and frequency knots $t_k$, the modality-specific penalty is
\begin{equation}
  \begin{aligned}
  \mathcal{L}_{\mathrm{SIGReg}}^{m}
  &=
  \frac{1}{P}
  \sum_{p=1}^{P}
  N_m\sum_{k=1}^{K}w_k\\[-2pt]
  &\quad\times
  \left|
    \frac{1}{N_m}\sum_{z\in\mathcal{Z}_m}
    e^{\mathrm{i}t_k(a_p^m)^\top z}\right.\\[-2pt]
  &\left.\qquad
    -e^{-t_k^2/2}
  \right|^2 .
  \end{aligned}
  \label{eq:app-sigreg-modality}
\end{equation}
Image and SMILES positions are pooled across examples within their respective modality, but the two modalities are never pooled into one SIGReg statistic.
Because the statistic is computed from the clean branch, active mask tokens do not enter $\mathcal{Z}_m$.
The main-text regularizer in \Cref{eq:sigreg-total} is the weighted sum of these two penalties.

\subsection{Training-Stream Details}
\label{app:training-stream-details}

The unlabeled pool used for R1 and R2 includes molecular images from USPTO-Mol
(\url{https://huggingface.co/datasets/xiegeo/uspto-mol}) and structure figures
extracted from scientific papers.
The coefficients $\lambda_{\mathrm{sft}}$, $\lambda_{\mathrm{mse}}^m$, and $\lambda_{\mathrm{sig}}^m$ are fixed within a training run.
SFT and JEPA are estimated from separate mini-batches drawn from $\mathcal{D}_{\mathrm{sft}}$ and $\mathcal{D}_{\mathrm{jepa}}^{g}$, respectively, and may use different batch sizes; this changes the stochastic estimator but not the objectives in \Cref{eq:anchor-objective,eq:latent-self-training-objective}.
For the simplest accepted pseudo-label pool, $q(I,\hat S)=1$.
Algorithm~\ref{alg:ac-jepa-training} makes the reported refresh protocol explicit.
The common SFT-anchored, JEPA-warmed checkpoint $C_0$ generates the first pseudo-pair pool and initializes the round-1 student.
The trained round-1 checkpoint generates the second pool, but the round-2 student is reinitialized from $C_0$ rather than continued from the round-1 weights.

\paragraph{Synthetic-pair reuse pilot.}
\label{app:synthetic-reuse-pilot}
Our small-scale mixed-SFT implementation sampled synthetic pairs from MolParser-7M together with Real-SFT for categorical token supervision.
Relative to Real-SFT-only adaptation, it provided no improvement and was more prone to overfitting.
This is a qualitative observation motivating the choice to retain real-image token supervision and reuse synthetic pairs through latent prediction.
It is separate from the controlled label-corruption experiment, which uses MolGrapher-Synthetic-300K.
We do not assign a numerical endpoint or claim a complete 7M-example rerun or a matched-compute comparison for this pilot.

\begin{algorithm}[tb]
\caption{\method{} training and pseudo-label refresh}
\label{alg:ac-jepa-training}
\small
\raggedright
\textbf{Input}: common checkpoint $C_0$; reliable SFT pairs
$\mathcal D_{\mathrm{sft}}$; gold JEPA pairs $\mathcal D_{\mathrm{jepa}}^g$;
unlabeled images $\mathcal D_u$\\
\textbf{Output}: independently trained checkpoints $C_g,C_1,C_2$\\
\textbf{Notation}: $\ell_{\mathrm{MSE}}^m$ and
$\ell_{\mathrm{SIGReg}}^m$ are unweighted modality-specific terms
\begin{algorithmic}[1]
\STATE \textbf{procedure} $\textsc{Train}(C_0,\mathcal D_p,\eta,T)$
\STATE Restore $C_0$ and initialize a run-specific optimizer
\FOR{$t=1,\ldots,T$}
  \STATE Sample $B_s\!\sim\!\mathcal D_{\mathrm{sft}}$ and
  $B_g\!\sim\!\mathcal D_{\mathrm{jepa}}^g$
  \STATE $\mathcal L\leftarrow
  \lambda_{\mathrm{sft}}\mathcal L_{\mathrm{SFT}}(B_s)$
  \STATE $\mathcal Q\leftarrow\{(B_g,1)\}$
  \IF{$\mathcal D_p\neq\emptyset$}
    \STATE Sample $B_p\sim\mathcal D_p$ and add
    $(B_p,\eta q(B_p))$ to $\mathcal Q$
  \ENDIF
  \FOR{each $(B,w)\in\mathcal Q$}
    \STATE Construct clean and structured-masked views of $B$
    \STATE $Z_c\leftarrow g_\phi(H_\theta^\ell(\mathbf V_{\mathrm{clean}}))$
    \STATE $Z_t\leftarrow\operatorname{sg}(Z_c)$
    \STATE $Z_p\leftarrow
    p_\psi(g_\phi(H_\theta^\ell(\mathbf V_{\mathrm{mask}})))$
    \STATE $\mathcal L_{\mathrm{pred}}\leftarrow
    \sum_m\lambda_{\mathrm{mse}}^m
    \ell_{\mathrm{MSE}}^m(Z_p,Z_t)$
    \STATE $\mathcal L_{\mathrm{reg}}\leftarrow
    \sum_m\lambda_{\mathrm{sig}}^m
    \ell_{\mathrm{SIGReg}}^m(Z_c)$
    \STATE $\mathcal L\leftarrow\mathcal L+
    w(\mathcal L_{\mathrm{pred}}+\mathcal L_{\mathrm{reg}})$
  \ENDFOR
  \STATE Update $(\theta,\phi,\psi)$ using $\nabla\mathcal L$
\ENDFOR
\STATE \textbf{return} the trained checkpoint
\STATE \textbf{end procedure}
\STATE $C_g\leftarrow\textsc{Train}(C_0,\emptyset,0,T_g)$
\STATE $\mathcal D_p^{(1)}\leftarrow
\{(I,\operatorname{Decode}(C_0,I)):I\in\mathcal D_u\}$
\STATE $C_1\leftarrow
\textsc{Train}(C_0,\mathcal D_p^{(1)},\eta_1,T_1)$
\STATE $\mathcal D_p^{(2)}\leftarrow
\{(I,\operatorname{Decode}(C_1,I)):I\in\mathcal D_u\}$
\STATE $C_2\leftarrow
\textsc{Train}(C_0,\mathcal D_p^{(2)},\eta_2,T_2)$
\STATE \textbf{return} $C_g,C_1,C_2$
\end{algorithmic}
\end{algorithm}

\subsection{Supporting Analysis and Proofs for Task Readout Stability and Pseudo-Label Noise}
\label{app:conditional-analysis-proofs}

This appendix proves the conditional-readout and parameter-influence results in the main text and makes their assumptions explicit.
The first analysis holds a jointly trained model fixed, whereas the second is local to a counterfactual clean-pair solution anchored by the gold SFT and JEPA objectives.
It does not assert global robustness or compare loss functions independently of representation geometry.

\subsubsection{From Symbolic JEPA Loss to Conditional Residual}
\label{app:predictive-risk-link}

Sample $i$ from the masked SMILES positions under the training data and masking distribution.
Write $Z_i=Z_{\mathrm{tgt},i}$, $\widehat Z_i=Z_{\mathrm{pred},i}$, and let $C_i$ contain the corrupted multimodal latent sequence and mask-position information available to the predictor.
The conditional mean $m_i(C_i)=\mathbb E[Z_i\mid C_i]$ gives the exact orthogonal decomposition of the unweighted symbolic MSE component,
\begin{equation}
  \begin{aligned}
  \mathcal L_{\mathrm{MSE}}^{\mathrm{smi}}
  &=
  \mathbb E\|Z_i-\widehat Z_i\|_2^2\\
  &=
  \underbrace{\mathbb E\|Z_i-m_i(C_i)\|_2^2}
             _{\overline R_{\mathrm J}^{\mathrm{smi}}}
  +
  \underbrace{\mathbb E\|m_i(C_i)-\widehat Z_i\|_2^2}
             _{\mathcal E_{\mathrm{pred}}^{\mathrm{smi}}}.
  \end{aligned}
  \label{eq:app-predictive-risk-decomposition}
\end{equation}
The cross term vanishes because
$\mathbb E[Z_i-m_i(C_i)\mid C_i]=0$.
Thus $\overline R_{\mathrm J}^{\mathrm{smi}}$ is the Bayes residual of symbolic latent prediction and $\mathcal E_{\mathrm{pred}}^{\mathrm{smi}}$ measures predictor approximation and optimization error.
Stop-gradient changes the update path through $Z_i$ but not this forward-risk identity.

Consider two checkpoints evaluated under the same data and masking distribution.
Let
$\Delta_{\mathrm J}
=
\mathcal L_{\mathrm{MSE},0}^{\mathrm{smi}}
-
\mathcal L_{\mathrm{MSE},1}^{\mathrm{smi}}$,
and assume that predictor warm-up gives
$\mathcal E_{\mathrm{pred},0}^{\mathrm{smi}}\leq\epsilon_{\mathrm p}$.
Since predictor excess risk is nonnegative, \Cref{eq:app-predictive-risk-decomposition} yields
\begin{equation}
  \overline R_{\mathrm J,0}^{\mathrm{smi}}
  -
  \overline R_{\mathrm J,1}^{\mathrm{smi}}
  =
  \Delta_{\mathrm J}
  -
  \mathcal E_{\mathrm{pred},0}^{\mathrm{smi}}
  +
  \mathcal E_{\mathrm{pred},1}^{\mathrm{smi}}
  \geq
  \Delta_{\mathrm J}-\epsilon_{\mathrm p}.
  \label{eq:app-residual-reduction}
\end{equation}
Hence a symbolic-MSE decrease larger than the initial predictor excess reduces the Bayes latent residual.
The statement concerns the population objective at the two checkpoints; it does not require every stochastic update to decrease the residual.

To connect this average residual to \Cref{thm:task-anchored-readout}, define
$m_{i,s^+}(C_i)=\mathbb E[Z_i\mid C_i,S_{i+1}=s^+]$.
The law of total variance gives
\begin{equation}
  \begin{aligned}
  \overline R_{\mathrm J}^{\mathrm{smi}}
  &=
  \mathbb E\|Z_i-m_{i,S_{i+1}}(C_i)\|_2^2
  \\
  &\quad+
  \mathbb E\|m_{i,S_{i+1}}(C_i)-m_i(C_i)\|_2^2 .
  \end{aligned}
  \label{eq:app-class-residual-decomposition}
\end{equation}
Under the locally affine latent interface
$Z_i=J_iH_i+a_i$, the first term equals
$\mathbb E[R_{\mathrm J,i}(C_i)]$; for a nonlinear interface this equality holds to first order with the corresponding Taylor remainder.
Suppose the overlap constant $\kappa$ is uniform over the analyzed strata and the SFT anchor maintains
$\mu_{i,s^-}(C_i)\geq\mu_0>0$ for a fixed competitor assignment.
Applying \Cref{thm:task-anchored-readout}, averaging over contexts and next-token classes, and using concavity of $x/(x+\mu_0^2)$ gives
\begin{equation}
  \Pr[\Delta_{i,s^-}(H_i)\leq0]
  \leq
  \frac{\kappa\,\overline R_{\mathrm J}^{\mathrm{smi}}}
       {\kappa\,\overline R_{\mathrm J}^{\mathrm{smi}}+\mu_0^2}.
  \label{eq:app-average-readout-bound}
\end{equation}
Together, \Cref{eq:app-residual-reduction,eq:app-average-readout-bound} connect a decrease in the actual symbolic JEPA objective to a tighter averaged readout bound under predictor tracking, local decoder--JEPA overlap, and SFT margin preservation.
Modality-wise SIGReg excludes the trivial collapsed route to a small latent residual but does not replace these conditions.
Under these conditions, the SFT and JEPA objectives play complementary roles.
SFT maintains the class-conditional mean margin, while a reduction in symbolic JEPA risk reduces context-unpredictable latent variation.
Decoder--JEPA overlap converts this reduction into a tighter bound on the probability that hidden-state fluctuations cross the decision boundary.

\paragraph{Image-to-symbol transfer.}
The preceding argument uses symbolic MSE because the corresponding latent state directly supports the causal decoder, but image-side prediction can also reduce symbolic decision variation through cross-modal fusion.
Let $V$ and $S$ denote image- and symbol-side intermediate states and locally write the fused symbolic state as $H_i=F_i(V,S)$.
Around the conditional mean, its centered residual admits the first-order decomposition
\begin{equation}
  U_i
  =
  A_iU_{\mathrm{img}}
  +
  B_iU_{\mathrm{smi}}
  +
  r_i,
  \label{eq:app-cross-modal-residual}
\end{equation}
where $A_i$ is the image-to-symbol Jacobian induced by cross-modal interaction, $B_i$ is the symbol-side Jacobian, and $r_i$ is the linearization remainder.
For modality-specific latent interfaces $J_{\mathrm{img}}$ and $J_{\mathrm{smi}}$, define
$R_{\mathrm J}^{m}=\mathbb E[\|J_mU_m\|_2^2\mid C_i,S_{i+1}=s^+]$.
If
$A_i^\top d_{i,s^-}d_{i,s^-}^\top A_i\preceq
\kappa_{\mathrm{img}}J_{\mathrm{img}}^\top J_{\mathrm{img}}$
and the analogous symbol-side overlap holds with $\kappa_{\mathrm{smi}}$, then
\begin{equation}
  v_{i,s^-}(C_i)
  \leq
  \kappa_{\mathrm{img}}R_{\mathrm J}^{\mathrm{img}}
  +
  \kappa_{\mathrm{smi}}R_{\mathrm J}^{\mathrm{smi}}
  +
  \epsilon_{\mathrm{cross}}
  +
  \epsilon_{\mathrm{lin}},
  \label{eq:app-cross-modal-variance}
\end{equation}
where $\epsilon_{\mathrm{cross}}$ bounds cross-modal covariance terms and
$\epsilon_{\mathrm{lin}}$ bounds the contribution of $r_i$.
Applying the decomposition in \Cref{eq:app-predictive-risk-decomposition} to masked image positions gives the analogous image-side loss--residual link; connecting its population average to $R_{\mathrm J}^{\mathrm{img}}$ additionally requires the paired sampling and local transfer assumptions above.
Thus image-side stabilization can tighten the decision-reversal bound in \Cref{thm:task-anchored-readout} when cross-modal transfer reaches the decoder-relevant subspace, the remaining terms are controlled, and SFT preserves the margin.
With locally fixed attention scores, $A_i$ reduces to attention-weighted value transfer; score variation is absorbed into the remainder.
The fact that both modalities describe the same entity makes the required alignment plausible but does not guarantee it: no improvement follows if the decoder ignores the image subspace or if correlated or systematic errors dominate.
Because $C_i$ and $H_i$ in \Cref{thm:task-anchored-readout} are already multimodal, this decomposition makes an implicit pathway explicit rather than introducing a separate guarantee.

\begin{proof}[Proof of \Cref{thm:task-anchored-readout}]
Fix $i\in\mathcal{M}_{\mathrm{smi}}$ and condition on $(C_i,S_{i+1}=s^+)$.
The state $H_i$ is aligned with $S_i$, whereas the causal readout at the same decoder position predicts $S_{i+1}$.
With the other causal states fixed, write
$U_i=H_i-\mathbb E[H_i\mid C_i,S_{i+1}=s^+]$.
Under the locally affine readout,
\begin{equation}
  \begin{aligned}
  \Delta_{i,s^-}(H_i)-\mu_{i,s^-}(C_i)
  &=d_{i,s^-}^{\top}U_i,\\
  v_{i,s^-}(C_i)
  &=
  \mathbb E\!\left[
  (d_{i,s^-}^{\top}U_i)^2
  \,\middle|\,
  \substack{C_i,\\S_{i+1}=s^+}
  \right].
  \end{aligned}
\end{equation}
Cantelli's inequality gives
\begin{equation}
  \Pr\!\left[
  \Delta_{i,s^-}(H_i)\leq0
  \,\middle|\,
  \substack{C_i,\\S_{i+1}=s^+}
  \right]
  \leq
  \frac{v_{i,s^-}(C_i)}
       {v_{i,s^-}(C_i)+\mu_{i,s^-}(C_i)^2}.
\end{equation}
Moreover, the decoder--JEPA overlap condition
$d_{i,s^-}d_{i,s^-}^{\top}\preceq\kappa J_i^\top J_i$ implies
\begin{align}
  v_{i,s^-}(C_i)
  &=
  \mathbb E[
    U_i^\top d_{i,s^-}d_{i,s^-}^{\top}U_i
    \mid C_i,S_{i+1}=s^+]\\
  &\leq
  \kappa\mathbb E[\|J_iU_i\|_2^2\mid C_i,S_{i+1}=s^+].
\end{align}
The function $v/(v+\mu^2)$ is increasing in $v$ and decreasing in positive $\mu$, which proves the final comparison.
For a nonlinear JEPA interface, $J_iU_i$ is the first-order projected residual; the theorem therefore applies only where the local linearization and overlap condition are valid.
\end{proof}

\subsubsection{Pseudo-label Residuals and Parameter Influence}
\label{app:pseudolabel-output-influence}

The second result compares two training problems that use the same images and differ only in the symbolic sequence attached to each pseudo-pair.
Let
\begin{equation}
  \begin{aligned}
  \mathcal L_{\mathrm{clean}}(\theta)
  &=
  \mathcal L_{\mathrm{anchor}}(\theta)
  +
  \frac{\lambda}{n}\sum_{j=1}^{n}
  \ell_j(\theta;S_j),\\
  \mathcal L_{\mathrm{pseudo}}(\theta)
  &=
  \mathcal L_{\mathrm{anchor}}(\theta)
  +
  \frac{\lambda}{n}\sum_{j=1}^{n}
  \ell_j(\theta;\widehat S_j),\\
  \ell_j(\theta;S)
  &=
  \frac12\|r_j(\theta;S)\|_2^2 .
  \end{aligned}
  \label{eq:app-clean-pseudo-objectives}
\end{equation}
$\mathcal L_{\mathrm{anchor}}$ contains the gold-pair SFT and JEPA streams, and $\theta^\star$ is a local stationary point of $\mathcal L_{\mathrm{clean}}$.
The unobserved sequences $S_j$ define a counterfactual reference and are not required by the training algorithm.
Using the same images in both objectives separates pseudo-label corruption from the useful signal that can arise simply from adding open-environment diagrams.
The factor $1/2$ only absorbs the conventional factor of two into $\lambda$.

For a fixed image, mask realization, and predicted latent position, write
\begin{equation}
  \begin{aligned}
  r_j^{\mathrm c}
  &=
  Z_{\mathrm{pred},j}(\theta^\star;S_j)
  -
  Z_{\mathrm{tgt},j}(\theta^\star;S_j),\\
  r_j^{\mathrm p}
  &=
  Z_{\mathrm{pred},j}(\theta^\star;\widehat S_j)
  -
  Z_{\mathrm{tgt},j}(\theta^\star;\widehat S_j),\\
  e_j
  &=
  r_j^{\mathrm p}-r_j^{\mathrm c}
  =
  e_j^{\mathrm{ctx}}-e_j^{\mathrm{tgt}},\\
  e_j^{\mathrm{ctx}}
  &=
  Z_{\mathrm{pred},j}(\theta^\star;\widehat S_j)
  \\
  &\quad-
  Z_{\mathrm{pred},j}(\theta^\star;S_j),\\
  e_j^{\mathrm{tgt}}
  &=
  Z_{\mathrm{tgt},j}(\theta^\star;\widehat S_j)
  \\
  &\quad-
  Z_{\mathrm{tgt},j}(\theta^\star;S_j).
  \end{aligned}
  \label{eq:app-effective-residual-perturbation}
\end{equation}
This decomposition is not a binary assignment of an error to the context or target route.
A changed pseudo-token can affect the visible predictor context in one mask realization and the clean target states of other positions in the same example.
The index $j$ may therefore include sampling over mask realizations and masked positions.

Let
$A_j=\partial r_j(\theta;S_j)/\partial\theta|_{\theta^\star}$ and
$\widehat A_j=\partial r_j(\theta;\widehat S_j)/\partial\theta|_{\theta^\star}$.
These derivatives follow the implemented MSE path, so the detached target contributes to the value of $e_j$ but not to either residual Jacobian.
The exact gradient difference for the squared residual is
\begin{equation}
  \begin{aligned}
  \nabla_\theta\ell_j(\theta^\star;\widehat S_j)
  -
  \nabla_\theta\ell_j(\theta^\star;S_j)
  &=
  \widehat A_j^\top r_j^{\mathrm p}
  -
  A_j^\top r_j^{\mathrm c}\\
  &=
  A_j^\top e_j
  +
  \rho_j,\\
  \rho_j
  &=
  (\widehat A_j-A_j)^\top r_j^{\mathrm p}.
  \end{aligned}
  \label{eq:app-residual-gradient-difference}
\end{equation}
The leading term pulls the explicit latent residual perturbation back to parameter space.
The remainder $\rho_j$ records the change of the predictor Jacobian caused by the pseudo sequence and becomes small when the local shared-Jacobian approximation is accurate or the perturbed residual is small.

Let $K=\nabla_\theta^2\mathcal L_{\mathrm{clean}}(\theta^\star)$ be positive definite on the local parameter subspace under analysis, and set $u=\theta-\theta^\star$.
Neglecting $\rho_j$ and the change in curvature between the two objectives gives
\begin{equation}
  \mathcal L_{\mathrm{pseudo}}(\theta^\star+u)
  -
  \mathcal L_{\mathrm{pseudo}}(\theta^\star)
  \approx
  \frac12u^\top Ku
  +
  \frac{\lambda}{n}
  \sum_{j=1}^{n}e_j^\top A_j u .
  \label{eq:app-local-pseudo-objective}
\end{equation}
Its stationary displacement is
\begin{equation}
  u_{\mathrm{pseudo}}
  =
  -\frac{\lambda}{n}K^{-1}
  \sum_{j=1}^{n}A_j^\top e_j .
  \label{eq:app-parameter-displacement}
\end{equation}
This is a change in the shared model parameters, not an independently optimized hidden state.
Its effect on any clean example is mediated by that example's parameter-to-state Jacobian.

\begin{proof}[Proof of \Cref{thm:anchored-latent-self-training}]
For a clean evaluation context $x$, define
$J_x=\partial H_x/\partial\theta|_{\theta^\star}$ and
$d_{x,yk}=\nabla_{H_x}G_{x,yk}(\theta^\star)$.
The corresponding parameter-space margin gradient is
$q_{x,yk}=J_x^\top d_{x,yk}$.
Combining the first-order margin expansion with
\Cref{eq:app-parameter-displacement} gives
\begin{equation}
  \begin{aligned}
  G_{x,yk}^{\mathrm{lin}}
  &=
  m_{x,yk}
  +
  q_{x,yk}^\top u_{\mathrm{pseudo}}\\
  &=
  m_{x,yk}
  -
  \frac{\lambda}{n}\sum_{j=1}^{n}
  d_{x,yk}^\top J_xK^{-1}A_j^\top e_j\\
  &=
  m_{x,yk}
  -
  \lambda\overline\zeta_{x,yk},
  \qquad
  \overline\zeta_{x,yk}
  =
  \frac1n\sum_{j=1}^{n}\zeta_{x,yk,j}.
  \end{aligned}
  \label{eq:app-linearized-output-margin}
\end{equation}
Under the theorem assumptions,
$\mathbb E[\overline\zeta_{x,yk}]=\beta_{x,yk}$ and
$\operatorname{Var}(\overline\zeta_{x,yk})\leq\sigma_{x,yk}^2/n$.
When $m_{x,yk}>\lambda\beta_{x,yk}$, boundary crossing requires
\begin{equation}
  \overline\zeta_{x,yk}-\beta_{x,yk}
  \geq
  \frac{m_{x,yk}-\lambda\beta_{x,yk}}{\lambda}.
\end{equation}
Cantelli's inequality therefore yields
\begin{equation}
  \Pr[G_{x,yk}^{\mathrm{lin}}\leq0]
  \leq
  \frac{\sigma_{x,yk}^2/n}
       {\sigma_{x,yk}^2/n+
        (m_{x,yk}/\lambda-\beta_{x,yk})^2},
\end{equation}
which is equivalent to \Cref{eq:pseudo-output-stability}.
\end{proof}

For a multi-class next-token decision, applying the union bound over
$k\in\mathcal V_{\mathrm{dec}}\setminus\{y\}$ gives
\begin{equation}
  \Pr[\operatorname*{argmax}_{s}a_{x,s}(\theta)\neq y]
  \leq
  \sum_{k\neq y}
  \Pr[G_{x,yk}(\theta)\leq0].
  \label{eq:app-multiclass-pseudo-bound}
\end{equation}
Each term uses its own clean margin, output-influence bias, and variance; no independence across competitors is required.

\par\medskip\noindent\textbf{Approximation terms and SIGReg.}\par\nobreak
Equation~\eqref{eq:app-residual-gradient-difference} makes the shared-Jacobian approximation explicit.
Retaining its remainder changes \Cref{eq:app-parameter-displacement} by
$-\lambda K^{-1}n^{-1}\sum_j\rho_j$.
The pseudo sequence also changes the clean latents entering modality-wise SIGReg.
If $g_{\mathrm{sig}}$ denotes the difference between the pseudo-sequence and true-sequence SIGReg gradients, the total additional displacement is
\begin{equation}
  u_{\mathrm{rem}}
  =
  -K^{-1}
  \left(
  \frac{\lambda}{n}\sum_{j=1}^{n}\rho_j
  +
  g_{\mathrm{sig}}
  \right).
  \label{eq:app-pseudo-remainder-displacement}
\end{equation}
We keep these terms separate from $e_j$ because they do not describe a latent residual perturbation.
They can be incorporated into \Cref{thm:anchored-latent-self-training} by adding
$q_{x,yk}^\top u_{\mathrm{rem}}$ to the margin and controlling its bias and variance.
SIGReg can contribute stabilizing curvature to $K$, but its pseudo-label-dependent gradient is not assumed to be zero.

The nonlinear parameter-to-margin map contributes a Taylor remainder.
If the combined scalar remainder in
$G_{x,yk}=m_{x,yk}-\lambda\overline\zeta_{x,yk}+\mathcal R_{x,yk}$
satisfies $|\mathcal R_{x,yk}|\leq\tau_{x,yk}$, the same argument gives
\begin{equation}
  \Pr[G_{x,yk}\leq0]
  \leq
  \frac{\lambda^2\sigma_{x,yk}^2/n}
       {\lambda^2\sigma_{x,yk}^2/n+
        (m_{x,yk}-\tau_{x,yk}
         -\lambda\beta_{x,yk})^2}
  \label{eq:app-output-bound-with-remainder}
\end{equation}
whenever $m_{x,yk}>\tau_{x,yk}+\lambda\beta_{x,yk}$.

\paragraph{Anchor curvature.}
The parameter-space curvature makes the roles of SFT and JEPA explicit.
Let $W_i=\partial a_i/\partial\theta$ be the Jacobian of the clean SFT logits $a_i$ and let $p_i$ be their token distribution.
The cross-entropy Gauss--Newton term is
\begin{equation}
  K_{\mathrm{SFT},i}
  =
  W_i^\top
  [\operatorname{Diag}(p_i)-p_ip_i^\top]
  W_i
  \succeq0.
  \label{eq:app-sft-curvature}
\end{equation}
If $B_j$ is the parameter Jacobian of a clean JEPA residual, its squared loss contributes $B_j^\top B_j$.
A damped local Gauss--Newton model on an identifiable parameter subspace $\mathcal U$ is
\begin{equation}
  \begin{aligned}
  K
  &\approx
  \alpha_{\mathrm S}\sum_i K_{\mathrm{SFT},i}
  +
  \alpha_{\mathrm J}\sum_j B_j^\top B_j
  \\
  &\quad+
  K_{\mathrm{reg}}
  +
  \mu P_{\mathcal U},
  \\
  K_{\mathrm{reg}}
  &\succeq0,\qquad\mu>0.
  \end{aligned}
  \label{eq:app-anchor-curvature}
\end{equation}
where $K_{\mathrm{reg}}$ is a positive-semidefinite local approximation to regularization curvature and $P_{\mathcal U}$ projects onto $\mathcal U$.
The restriction is necessary because neural-network parameterizations contain flat and symmetry-related directions.
The theorem depends on curvature only along the pulled-back pseudo-error and evaluation-margin directions, rather than on global positive definiteness of the full network Hessian.
With mean-reduced objectives, more clean samples do not automatically scale $K$; they can nevertheless improve curvature coverage and reduce its estimation variance.

\paragraph{Bias, variance, and dependence.}
The harmful bias is
\begin{equation}
  \beta_{x,yk}
  =
  \mathbb E[
  d_{x,yk}^\top J_xK^{-1}A_j^\top e_j],
  \label{eq:app-output-influence-bias}
\end{equation}
which depends jointly on the latent residual perturbation, the training Jacobian, the anchored curvature, and the evaluation readout.
Consequently, $\mathbb E[e_j]=0$ is not sufficient when $A_j$ and $e_j$ are dependent.
The stronger condition $\mathbb E[e_j\mid A_j]=0$ implies $\beta_{x,yk}=0$ for fixed $x$, $K$, and $\theta^\star$, but model-generated pseudo-labels need not satisfy it.
The bias--variance statement in \Cref{thm:anchored-latent-self-training} therefore retains $\beta_{x,yk}$ instead of assuming exact centering.
If pseudo constraints are correlated, the variance of their average becomes
$n^{-2}\sum_{j,j'}\operatorname{Cov}(\zeta_j,\zeta_{j'})$ and need not decay as $1/n$.

\paragraph{Single-error comparison with direct pseudo-SFT.}
For one erroneous pseudo-token, let $\widetilde y$ be the erroneous alternative,
$G_{y\widetilde y}(\theta^\star)=m>0$, and
$q=\nabla_\theta G_{y\widetilde y}(\theta^\star)$.
One gradient step of size $\rho$ on the pseudo-JEPA perturbation relative to its clean-pair counterpart gives
\begin{equation}
  \theta_{\mathrm J}'
  =
  \theta^\star-\rho\lambda A^\top e,
  \qquad
  m_{\mathrm J}'
  \approx
  m-\rho\lambda q^\top A^\top e.
  \label{eq:app-jepa-single-error}
\end{equation}
In a binary local readout, let $p_y=\operatorname{sigmoid}(m)$.
Cross-entropy toward $\widetilde y$ has gradient $p_yq$, so direct pseudo-SFT gives
\begin{equation}
  \theta_{\mathrm S}'
  =
  \theta^\star-\rho\lambda p_yq,
  \qquad
  m_{\mathrm S}'
  \approx
  m-\rho\lambda p_y\|q\|_2^2.
  \label{eq:app-sft-single-error}
\end{equation}
Thus the wrong categorical target decreases the correct-versus-wrong margin in this local model.
The pseudo-JEPA effect instead depends on the signed alignment
$q^\top A^\top e$ and can either decrease or increase that margin.
For the corresponding binary probability,
\begin{equation}
  |p_y^{\mathrm J}-p_y|
  \leq
  \frac{\rho\lambda}{4}
  |q^\top A^\top e|
  \leq
  \frac{\rho\lambda}{4}
  \|Aq\|_2\|e\|_2.
  \label{eq:app-jepa-probability-bound}
\end{equation}

With the clean anchor retained, the local pseudo-JEPA equilibrium is
$u_{\mathrm J}=-\lambda K^{-1}A^\top e$.
A quadratic expansion of the erroneous pseudo-SFT term gives
\begin{equation}
  u_{\mathrm S}
  \approx
  -\lambda p_y
  [K+\lambda p_y(1-p_y)qq^\top]^{-1}q.
  \label{eq:app-sft-equilibrium-shift}
\end{equation}
The bracketed matrix is positive definite on $\mathcal U$, and
$q^\top u_{\mathrm S}<0$.
This comparison does not order the magnitudes of all JEPA and SFT errors.
It identifies their different routes to the output boundary: pseudo-SFT directly optimizes an erroneous categorical preference, whereas pseudo-JEPA first induces a continuous residual perturbation whose effect is filtered by $A$, $K^{-1}$, and the clean readout direction.

\subsubsection{Multimodal Evidence and Local Readout Stability}
\label{app:cross-modal-evidence-corollary}

\Cref{thm:task-anchored-readout} relates decoder-relevant latent variation to next-token stability, but it does not quantify how evidence from multiple modalities changes that variation.
We consider $M$ co-observed modalities that describe the same underlying entity or state.
For modality $m$, let $\mathcal G_m$ contain its reliable, target-relevant entries, and let $\mathcal N$ contain the visible noisy entries across the multimodal context.
Under a locally fixed attention map with score $\alpha_k$ and temperature $\tau>0$, define
\begin{equation}
  \begin{aligned}
  Z_m
  &=
  \sum_{g\in\mathcal G_m}e^{\alpha_g/\tau},
  &
  Z_{\mathcal N}
  &=
  \sum_{j\in\mathcal N}e^{\alpha_j/\tau},\\
  Z_{\mathcal G}^{(M)}
  &=
  \sum_{m=1}^{M}Z_m,
  &
  \Delta_M
  &=
  \log\frac{Z_{\mathcal G}^{(M)}}{Z_{\mathcal N}}.
  \end{aligned}
  \label{eq:app-multimodal-partitions}
\end{equation}
Let $a_k\propto e^{\alpha_k/\tau}$ be the corresponding attention weights.
The total noisy attention mass and its effective support are
\begin{equation}
  \beta_M
  =
  \sum_{j\in\mathcal N}a_j
  =
  \frac{1}{1+e^{\Delta_M}},
  \qquad
  n_{\mathrm{eff}}
  =
  \frac{\beta_M^2}{\sum_{j\in\mathcal N}a_j^2}.
  \label{eq:app-multimodal-noisy-mass}
\end{equation}

Suppose only the noisy values are perturbed,
$\widetilde v_j=v_j+\xi_j$ for $j\in\mathcal N$, where the $\xi_j$ are conditionally independent, centered, and satisfy
$\operatorname{Cov}(\xi_j\mid C_i,s^+)\preceq\sigma_v^2I$.
Let their attention-output perturbation be
$\widetilde o_M-o_M^\star=\sum_{j\in\mathcal N}a_j\xi_j$.
We locally decompose the correct-versus-competitor margin as
\begin{equation}
  \widetilde\Delta_{i,s^-}
  =
  \mu_{i,s^-}
  +
  \varepsilon_0
  +
  w^\top(\widetilde o_M-o_M^\star),
  \label{eq:app-multimodal-margin-decomposition}
\end{equation}
where $\mu_{i,s^-}>0$, $\varepsilon_0$ is conditionally centered and uncorrelated with the $\xi_j$, and
$\operatorname{Var}(\varepsilon_0\mid C_i,s^+)\leq v_0$.

\begin{corollary}[Multimodal evidence tightens the local readout bound]
\label{cor:cross-modal-evidence-stability}
Under the conditions above, define
\begin{equation}
  B_M^{\mathrm{vis}}
  =
  \frac{\|w\|_2^2\sigma_v^2}
       {n_{\mathrm{eff}}(1+e^{\Delta_M})^2}.
  \label{eq:app-multimodal-visible-bound}
\end{equation}
Then
\begin{equation}
  \Pr[
    \widetilde\Delta_{i,s^-}\leq0
    \mid C_i,s^+
  ]
  \leq
  \frac{v_0+B_M^{\mathrm{vis}}}
       {v_0+B_M^{\mathrm{vis}}+\mu_{i,s^-}^2}.
  \label{eq:app-multimodal-readout-bound}
\end{equation}
If an additional modality contributes reliable evidence mass $Z_{M+1}>0$ while the existing scores, noisy entries, and local readout map remain fixed, then
\begin{equation}
  \frac{B_{M+1}^{\mathrm{vis}}}{B_M^{\mathrm{vis}}}
  =
  \left[
  \frac{1+e^{\Delta_M}}
       {1+e^{\Delta_M}
       (1+Z_{M+1}/Z_{\mathcal G}^{(M)})}
  \right]^2
  <1.
  \label{eq:app-multimodal-contraction}
\end{equation}
If every modality supplies $Z_m\geq Z_{\min}>0$, then
\begin{equation}
  B_M^{\mathrm{vis}}
  \leq
  \frac{\|w\|_2^2\sigma_v^2}{n_{\mathrm{eff}}}
  \left(
  1+\frac{MZ_{\min}}{Z_{\mathcal N}}
  \right)^{-2}.
  \label{eq:app-multimodal-count-bound}
\end{equation}
Thus the local decision-reversal bound tightens monotonically with $M$ when added modalities preserve the mean margin, do not increase $v_0$, and contribute reliable evidence under the fixed local map.
\end{corollary}

The corollary characterizes multimodal redundancy as evidence quality rather than modality count alone.
Images, audio, language, numerical sensor streams, and radar measurements can provide complementary observations of the same underlying state.
When each added modality contributes aligned evidence, their aggregate partition mass increases $\Delta_M$ and reduces the variance transmitted from visible noisy entries.
This is the precise sense in which additional modalities can strengthen mutual verification and improve local noise tolerance.

The monotonic result is conditional.
A redundant, mismatched, systematically biased, or ignored modality need not preserve the mean margin or contribute reliable evidence mass.
The analysis also concerns a local next-token decision under fixed attention scores and centered independent perturbations.
It therefore does not guarantee that adding modalities monotonically improves full-sequence exact-match accuracy.
The image--symbol setting studied by AC-JEPA is the $M=2$ instance of this more general mechanism.

\begin{proof}
Conditional centering makes the attention-output perturbation mean-zero.
Write
$\Sigma_j=\operatorname{Cov}(\xi_j\mid C_i,s^+)$.
For any $w$,
\begin{equation}
  \begin{aligned}
  \mathrm{Var}\!\left[
    w^\top(\widetilde o_M-o_M^\star)
    \mid C_i,s^+
  \right]
  &=
  \sum_{j\in\mathcal N}a_j^2
  w^\top\Sigma_jw\\
  &\leq
  \|w\|_2^2\sigma_v^2
  \sum_{j\in\mathcal N}a_j^2\\
  &=
  \frac{\|w\|_2^2\sigma_v^2\beta_M^2}
       {n_{\mathrm{eff}}}
  =
  B_M^{\mathrm{vis}}.
  \end{aligned}
  \label{eq:app-multimodal-directional-variance}
\end{equation}
Adding the uncorrelated residual variance $v_0$ and applying the first inequality in
\Cref{thm:task-anchored-readout} proves
\Cref{eq:app-multimodal-readout-bound}.

Adding modality $M+1$ replaces
$Z_{\mathcal G}^{(M)}$ by
$Z_{\mathcal G}^{(M)}+Z_{M+1}$ and therefore gives
\begin{equation}
  \Delta_{M+1}
  =
  \Delta_M
  +
  \log\left(
    1+\frac{Z_{M+1}}{Z_{\mathcal G}^{(M)}}
  \right).
\end{equation}
All noisy attention weights are scaled by the same factor, so
$n_{\mathrm{eff}}$ remains unchanged.
Substitution into \Cref{eq:app-multimodal-visible-bound} yields
\Cref{eq:app-multimodal-contraction}.
Finally,
$Z_{\mathcal G}^{(M)}=\sum_m Z_m\geq MZ_{\min}$,
which proves \Cref{eq:app-multimodal-count-bound}.
\end{proof}

\end{document}